%% file: toto.tex
\documentclass[letterpaper, 10 pt, conference]{ieeeconf}

\IEEEoverridecommandlockouts   
\usepackage[binary-units=true]{siunitx}
\usepackage{times}
\usepackage{amsmath,amssymb}
\usepackage{accents}
\usepackage{graphicx}
\usepackage{bm}
\let\labelindent\relax
\usepackage[inline]{enumitem}
\usepackage{array}
\usepackage{multicol}
\usepackage{multirow}
\usepackage{tabulary}
\usepackage{longtable}
\usepackage{booktabs}
\usepackage{caption}
\usepackage{subcaption}

\usepackage{todonotes}
\usepackage{cite}
\usepackage[hidelinks]{hyperref}
\usepackage{algorithm}% http://ctan.org/pkg/algorithm
\usepackage{algpseudocode}% http://ctan.org/pkg/algorithmicx
\usepackage{breqn}
\usepackage{tcolorbox}
\usepackage{mathtools}
\usepackage{tikz}
\usetikzlibrary{arrows}
\usetikzlibrary{calc}
\usetikzlibrary{arrows,decorations.markings}
\usepackage{nomencl}
\usepackage{balance}

\usetikzlibrary{shadows,trees,calc,arrows,shapes,positioning}
\allowdisplaybreaks

\input{./commands.tex}

\title{\LARGE \bf Time-Optimal Control via Heaviside Step-Function Approximation}
\author{Kai Pfeiffer$^1$, Quang-Cuong Pham$^{1,2}$%
	\thanks{$^1$The authors are with the School of Mechanical and Aerospace Engineering, Nanyang Technological University, Singapore.}%
	\thanks{$^{2}$HP-NTU Digital Manufacturing Corporate Lab, Nanyang Technological University, Singapore.}
	\thanks{This research was conducted in collaboration with HP Inc. and supported by National Research Foundation (NRF) Singapore and the Singapore Government through the Industry Alignment Grant (I1801E0028).}
}

\begin{document}
	\maketitle 
	\thispagestyle{empty}
	\pagestyle{empty}
	
	\begin{abstract}
		Least-squares programming is a popular tool in robotics due to its simplicity and availability of open-source solvers. However, certain problems like sparse programming in the $\ell_0$- or $\ell_1$-norm for time-optimal control are not equivalently solvable. In this work, we propose a non-linear hierarchical least-squares programming (NL-HLSP) for time-optimal control of non-linear discrete dynamic systems. We use a continuous approximation of the heaviside step function with an additional term that avoids vanishing gradients. We use a simple discretization method by keeping states and controls piece-wise constant between discretization steps. This way, we obtain a comparatively easily implementable NL-HLSP in contrast to direct transcription approaches of optimal control. We show that the NL-HLSP indeed recovers the discrete time-optimal control in the limit for resting goal points. We confirm the results in simulation for linear and non-linear control scenarios.
	\end{abstract}

	\section{Introduction} 
	
	Time-optimal control (TOC) can be considered a powerful tool when as fast as possible task fulfillment of a dynamic system is desired. However, optimal control methods based on direct methods for problem discretiziation are not easily implementable~\cite{gpops} or their solution relies on proprietary software~\cite{gpopsii}. In this work we propose a non-linear hierarchical least-squares programming (NL-HLSP) that can be easily implemented ($\sim$20 lines of code if a non-linear solver and a task library are available) and that is provably convergent to the true discrete time-optimal control (DTOC) in the limit and for resting goal points (that is the system is able to remain at these points; for example a robot is not able to remain at a Cartesian point if at the same time the desired velocity is not zero). Furthermore, our method is applicable to dynamics of any form (for example the inverse dynamics equations popular in robotics due to their computational efficiency~\cite{carpentier2018}), unlike other trajectory optimizers~\cite{adrlCT} or differential dynamic programming methods~\cite{pavlov2021}.
	
	Optimal control is the problem of identifying a control $u$ such that a dynamic system with states $x$ is driven to a desired goal state $f_d$ while minimizing a defined cost on control $u$, the state $x$ or linear / non-linear functions of it. A desired goal is thereby a state that the system should end up in while considering constraints on the controls $u$ and / or the states $x$. This boils down to a constrained optimization problem with Ordinary Differential Equations (ODE), that describe the dynamic system behavior, and algebraic equations, that describe physical relations~\cite{biegler2002}.
	
	Analytic solutions to certain specifications of the continuous optimal control problem have been proposed.
	A solution to the linear quadratic regulator (LQR) with linear dynamics can be found in~\cite{kirk1970}. The authors in~\cite{goebel2005} propose a solution with limits on the controls.
	However, for complicated non-linear systems, and especially constrained problems, analytical solutions are usually too difficult or impossible to formulate. Instead, in order to be machine solvable, the original optimal control problem can be discretized and then be solved as a non-linear programming (NLP), commonly referred to as direct transcription method~\cite{Biral2016}. Different methods have been proposed which differ in the polynomials and collocation points (at which the functions are evaluated) that are used to approximate the continuous controls and states.
%	The use of Legendre-Gauss points has been proposed in~\cite{Benson2006} for accurate costate estimation.
	The authors in~\cite{Shivakumar2008} use Legendre-Gauss-Radau quadrature as it allows easy constraint formulation and is shown to possess high stability for systems of high order ODE's.
	Legendre-Gauss-Lobatto points however provide the smallest interpolation error in a least-squares sense~\cite{Benson2006}.
	The open-source matlab implementation GPOPS of a direct transcription method is described in~\cite{gpops}. However, the corresponding C++ implementation~\cite{gpopsii} is proprietary and furthermore relies on the proprietary NLP solver SNOPT~\cite{snopt}.
	
	A specific form of optimal control is time-optimal control~\cite{pontryagin1987}. Here, the cost function specifically aims to minimize the time at which a desired goal state is reached. The resulting control usually exhibits a `bang-bang' profile as a limit control for given bounds on the control inputs~\cite{bellman1956}.
	
	TOC is a complex problem, especially when the controls and states are considered at the same time, for example for whole-body robot trajectory optimization~\cite{Stryk1998}. 	
	One way to reduce its complexity is to reduce DTOC to a simpler problem by only considering the controls while the states are assumed to be known. This method, referred to as time-optimal path parametrization, has seen significant leaps in the recent past in terms of accuracy, convergence and computational complexity~\cite{Pham2018c,Pham2018b}. However, how to choose the states is oftentimes not entirely clear but might be taken for example from kinematic solutions or motion capture.
	
	Aside direct collocation methods, optimal control can also be discretized to reasonable accuracy for example by the explicit Euler-method~\cite{chen2012,Flasskamp2019,meduri2022}. In this case TOC can be considered a mixed-integer non-linear programming~\cite{Belotti2013} where the discrete optimal time is represented as an integer and the corresponding states and controls are continuous. Such problems, even in the linear case~\cite{Marcucci2020} or for example expressed as optimization problems with $\ell_0$-norm cost functions for sparsity enhancing linear regression~\cite{Candes2007}, are very expensive to solve due to their combinatorial nature. A simplified approach only considering a robot's end-effector position in a traveling salesman scenario has been treated in~\cite{Gentilini2013}.
	Another approach is to turn the $\ell_0$- into a weighted $\ell_1$-norm optimization problem. Thereby, different weights have been proposed in the literature~\cite{Candes2007,Cai2014,Zhou2016} which aim to represent the original problem as closely as possible. This approach has been borrowed for discrete linear dynamic systems for example in~\cite{chen2012}. A similar approach is followed in~\cite{Doetlinger2014}, confirming that $\ell_1$-norm optimization can provably retrieve an equivalent solution to DTOC in the linear case. 
	The authors in~\cite{Yang2020} propose a similar method based on a sliding window but formulated as an $\ell_2$-norm optimization problem. The window is iteratively shifted until the time-optimal control is contained within.	
	
    Control problem formulation as least-squares problems ($\ell_2$-norm) is oftentimes appropriate and sufficient in robot control and planning~\cite{vaillant:auro:2016,pfeiffer2017} and enjoys great popularity because of its simplicity and availability of open-source solvers~\cite{ceres}. A special formulation of least-squares programming is hierarchical least-squares programming~\cite{Escande2014}. Here, constraints and objectives can be further prioritized within such that a more efficient robot control can be achieved~\cite{pfeiffer2023}. Especially the separation of regularization tasks is oftentimes very helpful as the actual objectives can be fulfilled to higher accuracy as has been demonstrated in~\cite{pfeiffer2018}.
    However, there are problem settings that are not equivalently solvable in the $\ell_2$-norm. An example would be the above mentioned problem in the $\ell_1$-norm for DTOC~\cite{chen2012,Doetlinger2014}. In this work we propose a provably equivalent approximation of linear and non-linear DTOC in the $\ell_2$-norm based on the approximate heaviside function~\cite{berg1929heaviside}.
    The heaviside function has been treated in different works. The authors in~\cite{Wang2017} use a heaviside step-function approximation in order to indicate mechanical stress violations in topology optimization. The authors in~\cite{Zhou2020} directly optimize over the discontinuous Heaviside step-function and propose an appropriate Newton's method to do so. 
    
    Our contribution is therefore threefold:
    \begin{itemize}
    	\item Approximate discrete time-optimal control (ADTOC) as an easily implementable NL-HLSP that can be solved by off the shelf non-linear least-squares solvers.
       	\item Provable equivalence to true DTOC.
    	\item Applicability to both linear and non-linear systems.
    \end{itemize}
    
    First, we outline continuous TOC and our discretization of it, see Sec.~\ref{sec:prob}. Secondly, we describe ADTOC and the weight function based on the heaviside step-function approximation (Sec.~\ref{sec:aprxtopm}). In Sec.~\ref{sec:conv} we show the equivalence to  true DTOC in the limit and for resting goal points. Lastly, the algorithm is evaluated for linear and non-linear discrete dynamic systems (Sec.~\ref{sec:eval}).
	
	\section{Problem definition} 
	\label{sec:prob} 
	
	In this work, we consider TOC of the form
\begin{align}
	\mini_{x,u,T} \qquad &\int^T_0 dt \label{eq:conttopm}\tag{TOC}\\
	\text{s.t}\qquad 
	&\dot{x}(t) = f_{\text{dyn}}(x(t),u(t)) \nonumber\\
	& f_{\text{ter}}(x(T)) = 0\nonumber\\
	&x(t),u(t) \in \Omega\nonumber
\end{align}
The states $x(t)\in\mathbb{R}^{n_x}$ and $u(t)\in\mathbb{R}^{n_u}$ are continuous in time $t$.
The goal is to minimize the time $T$ that it takes to reach the terminal (ter.) state $f_{\text{ter}}(x(T)) = 0$.
We define
\begin{align}
	f_{\text{ter}}(x(t)) \coloneqq f_{\text{task}}(x(t)) - f_d(t) 
\end{align}
where $f_{\text{task}}(x(t))$ is some task function and  $f_d$ is the corresponding desired value. 
The dynamics $f_{\text{dyn}}$ determine the behavior of the control system by some possibly non-linear relationship. $\Omega$ is a constraint polytope which both the states and controls are constrained to.

The problem can be discretized for example by direct transcription methods. These methods provide an equivalent solution to the original continuous optimal control problem~\eqref{eq:conttopm} with strong convergence and stability properties~\cite{Shivakumar2008}. However, they are not easily implementable or rely on proprietary software. We therefore discretize our problem by the explicit Euler method by assuming  piece-wise constant states and controls over each discretization instance $i = 0,\dots, N-1$~\cite{chen2012,Flasskamp2019,meduri2022} (in contrast to higher order polynomials as in direct transcription methods).  We refer to the following optimization problem as the `true discrete time-optimal control' (DTOC) throughout the remainder of the paper:
\begin{align}
	\mini_{x,u,N^*} \qquad &T = N^*\Delta t
	\label{eq:topm}\tag{DTOC}\\
	\text{s.t}\qquad 
	&f_{\text{dyn}}(x(0), u(0),x(1),...,u(N-1),x(N)) = 0 \nonumber\\
	& f_{\text{ter}}(x(N^*+1),x) = 0\nonumber\\
	& x(0) = x_0\nonumber\\
	&x(i+1),u(i) \in \Omega\nonumber, \quad i = 0,\dots, N-1\nonumber
\end{align}
$\Delta t$ is the discretization time step. 
$N$ is the number of collocation points of the discrete problem. $N^*$ is the last time step at which $f_{\text{ter}}(x(N^*),x)\neq 0$ while at the consequent step we have $f_{\text{ter}}(x(N^*+1),x) = 0$. The dependence $f_{\text{ter}}(x(i),x)$ indicates that $f_{\text{ter}}$ is necessarily dependent on $x(i)$ but possibly also from $x$ (for example in case of finite differences or Euler integration schemes).
The discrete states $x$ and controls $u$ are defined as
\begin{align}
	&x \coloneqq \BIN x(1)^T& \dots&  x(N)^T\BOUT^T \\
	&u \coloneqq \BIN u(0)^T& \dots& u(N-1)^T\BOUT^T\nonumber
\end{align}
The initial state is given by $x(0) = x_0\in\mathbb{R}^{n_x}$. Note that we have generalized the dynamics $f_{\text{dyn}}$ without explicit dependence on the state derivatives. This can be advantageous in robotics if the Newton-Euler equations are used (see also Sec.~\ref{sec:manip}).

\ref{eq:topm} can be solved by means of $\ell_0$ optimization. However, due to its combinatorial nature an approximation based on the $\ell_1$-norm is usually solved instead. Different weight functions have been proposed which aim to approximate the original $\ell_0$-problem as close as possible. However, such $\ell_1$-problems can not be equivalently solved by least-squares programming. The reason is that the Hessian of the Taylor approximation of the problem is always positive-definite and not zero as required by $\ell_1$ programming ($\ell_1$ problems can be for example be solved by quadratic programming, or specialized linear programming solvers). Least-squares programming is a very popular tool due to its simplicity and availability of open-source solvers. Furthermore, it allows for hierarchical optimization, a tool that has gained considerable attention in the recent past, especially in the context of robot control. In this work we therefore propose a  NL-HLSP for TOC which is applicable both to linear and non-linear dynamic systems and task functions.

\section{Approximate discrete time-optimal control via heaviside step-function approximation}
\label{sec:aprxtopm}

In this section we recast \ref{eq:topm} into an NL-HLSP of the following form:
	\begin{align}
	\min_{{z},{v}_l} \quad & \frac{1}{2} \left\|{v}_l\right\|^2 \qquad\quad\hspace{8pt}l = 1,\dots,p
	\label{eq:NL-HLSP}\tag{NL-HLSP}\\
	\mbox{s.t.} \quad & {f}_l({z})  \leqq \hspace{3pt} v_l\nonumber\\
	&\uf_{l-1}({z})  \leqq  {\uv}_{l-1}^* \nonumber
\end{align}
Each of the $p$ priority levels contains constraints of the form $f_l(z) \leqq v_l$. The symbol $\leqq$ indicates both equality and inequality constraints. $v_l$ is a slack variable of level $l$ that is minimized in a least-squares sense while being subject to the constraints of the levels $1$ to $l$ while not increasing the already identified optimal slacks $\uv_{l-1}^* = \BIN v_{1}^{*T} & \cdots & v_{l-1}^{*T}\BOUT^T$ of the previous levels $1$ to $l-1$~\cite{escande2013companion,pfeiffer2018}. This problem can be solved by sequential hierarchical least-squares programming (S-HLSP) as proposed in~\cite{pfeiffer2023,pfeiffer2023b}.

Concretely, ADTOC as an \ref{eq:NL-HLSP} is given in Tab.~\ref{tab:topmaprx}. The variable vector is chosen as $z = \BIN x^T & u^T & N^*\BOUT$.
\begin{table}[tp!]
	\vspace{5pt}
	\centering
	\setlength{\tabcolsep}{1pt}
	\setlength{\extrarowheight}{1pt}
	\begin{tabular}{@{} ccccc @{}}  
		\toprule
		$l$ & $f_l(x,u,N^*) \leqq v_l$ \\
		\midrule
		& with $i = 0,\dots,N-1$\\
		1 & 
		$\BIN 
		u_i - u_{\max}\\
		u_{\min} - u_i\\
		\BOUT
		\leq v_{1,i+1}$ \\ 
		1 & 	$f_{\text{dyn}}(x_0,u_0,...,u_{N-1},x_{N}) = v_1 $\\
		2 & $\BIN 
		N^*\Delta t\\
		w(0,N^*) f_{\text{ter}}(x(0),x)\\
		\vdots \\
		w(N-1,N^*) f_{\text{ter}}(x(N-1),x)
		\BOUT=v_2$\\
		3 & Regularization of $x$, $u$, $N^*$\\
		\bottomrule
	\end{tabular}
	\caption{ADTOC}
	\label{tab:topmaprx}
\end{table}
For better readability we abbreviate variables $b$ at time $t$ as $b_t\coloneqq b(t)$.
ADTOC (Tab.~\ref{tab:topmaprx}) has three levels $p=3$. On the first priority level we define control constraints with minimum and maximum control values $u_{\text{min}}$ and $u_{\text{max}}$, respectively. Furthermore, the dynamics constraints $f_{\text{dyn}}$ are defined.

The second level contains the approximation of TOC. 
First, it contains the time minimization term $N^*\Delta t$. Furthermore,
for each discretization step $i=1,\dots,N$ we introduce the terminal constraint $f_{\text{ter}}(x(i))$ weighted by the function 
\begin{equation}
	w(i,N^*) =
	h(i,N^*)(i-N^*+1)^k
	\label{eq:w}
\end{equation}
The term $(i-N^*+1)^k$ is outlined further down this section. 
$h(i,N^*)$ is a differentiable approximation of the heaviside step function~\cite{berg1929heaviside}
\begin{equation}
	h(i,N^*) =
	(0.5 + 0.5 \tanh(k(i - N^*)))
\end{equation} 
The parameter $k>0$ regulates the steepness of the step approximation. Specfically, for $k\rightarrow \infty$ we have $h(i<N^*,N^*)=0$, $h(N^*,N^*)=0.5$ and $h(i>N^*,N^*)=1$ for some $i\in\mathbb{R}$ and $N^*\in\mathbb{R}$. The same holds for the function $w$. 
Instead of using the infinite limit we define a more numerically sensible threshold $k_{\epsilon}$ as follows.
\begin{definition}[Limit $k_{\epsilon}$]
	\label{def:limitk}
	The limit $k_{\epsilon}$ fulfills the bracket condition $h(i-1,i-0.5) = \epsilon$ and $h(i,i-0.5) = 1 - \epsilon$ for any $i\in\mathbb{Z}_{\geq 0}$. $k_{\epsilon}$ is then given by
	\begin{equation}
		k_{\epsilon} = 2\text{arctanh}\left(\frac{0.5 - \epsilon}{0.5}\right) 
		\label{eq:keps}
	\end{equation}
	$\epsilon$ signals a numerical threshold. Notably, $k_{\epsilon}$ is independent of $\Delta t$.
\end{definition}
%The above ensures that unless $N^*$ is exactly $i + 0.5$ for some $i\in\mathbb{Z}_{\geq 0}$, we have $w(i < i-0.5,i-0.5)<\epsilon$ or $w(i > i-0.5, i-0.5)>1-\epsilon$.
Note that while $k_{\epsilon}$ is independent of $\Delta t$ we have $k\rightarrow\infty$ for $\epsilon\rightarrow 0$. The parameter $\epsilon$ therefore needs to be chosen to a reasonable value compromising between the desired accuracy of ADTOC and numerical stability. Throughout this work we speak of `vanishing' functions even though this truly applies only in the limit $\epsilon\rightarrow 0$. 

Let us now further explain the term $(i-N^*+1)^k$ in the weight function $w(i,N^*)$~\eqref{eq:w} by looking at the derivative of $w(i,N^*)$ with respect to ${N}^*$
\begin{align}
	&\nabla_{N^*} w(i,N^*) = \label{eq:nabw}\\ 
	&\nabla_{N^*}h(i,N^*)(i-N^*+1)^k - kh(i,N^*)(i-N^*+1)^{k-1} \nonumber
\end{align}	
The derivative of the heaviside approximation is given by
\begin{align}
	\nabla_{N^*}h(i,N^*)= & -0.5 k    (1 - \tanh(k(i - N^*))^2) \label{eq:nabh}
\end{align}
It can be observed that for $k = k_{\epsilon}$ the derivative $\nabla_{N^*}h(i,N^*)$ vanishes for any $i\in\mathbb{Z}\setminus N^*$. In contrast, $\nabla_{N^*} w(i,N^*) \neq 0$ for any $i\in\mathbb{Z}_{>N^*}$ due to the second term of~\eqref{eq:nabw}.
This has numerical advantages as we further explain in Sec.~\ref{sec:cohn}.

\begin{figure}[tp!]
	\vspace{5pt}
	\includegraphics[width=0.8\columnwidth]{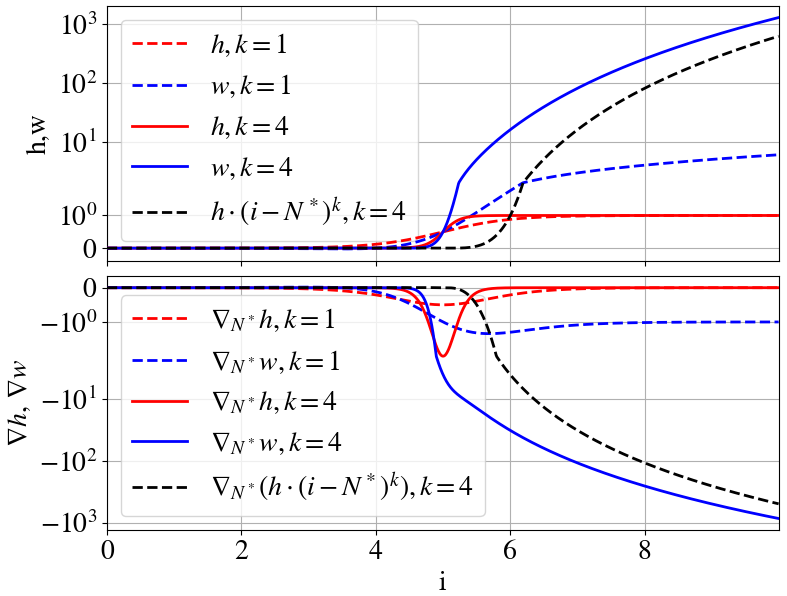}
	\centering
	\caption{Function and gradient values of $h(i,N^*)$, $w(i,N^*)$ and $h(i,N^*)(i-N^*)$ for $N^*=5$ and $k=1$ and $k=4$.}
	\label{fig:weight}
\end{figure}
A depiction of $w$ and $h$ and its derivatives is given in Fig.~\ref{fig:weight}.

It can be observed that both the time minimization term and the terminal constraints are not separated as constraint and objective as for \ref{eq:topm}.  
This means that we have turned the original TOC problem into a weighted optimization problem with weights $w$ on the terminal constraints $f_{\text{ter}}$. However, we show in Sec.~\ref{sec:conv} that with our choice of weights we can indeed recover \ref{eq:topm} in the limit $\Delta t\rightarrow 0$, $k=k_{\epsilon}$ and for resting goal points.

The hierarchical problem formulation allows us to define the regularization task on a separate level $l=3$ from the actual objective of finding the time-optimal control on level 2. The above problem can also be defined as a conventional constrained least-squares programming ($p=2$) by defining the (often desirable) regularization task on the second level with a small weight. However, this would negatively influence the task performance of the second level~\cite{pfeiffer2018}.

\section{Convergence of the approximate discrete time-optimal control problem}
\label{sec:conv}

% more subsections, reverse order of theorems

In this section we investigate the convergence behavior of ADTOC (Tab.~\ref{tab:topmaprx}). Similarly to~\cite{biegler2002} we proceed from a numerical viewpoint. We show that the true time-optimal control $\hat{u}$ and $\hat{x}$ poses a global minimum / first-order accumulation point / Karush-Kuhn-Tucker (KKT) point to Tab.~\ref{tab:topmaprx}. For this we look at following simplified optimization problem by assuming $x,u\in\Omega$ where $\Omega$ is the feasible constraint polytope with respect to the control, state and dynamics constraints:
\begin{align}
	\mini_{x,u\in\Omega,N^*} \qquad &\frac{1}{2}\left\Vert  \BIN N^{*}\Delta t\\
	w(0,N^*) f_{\text{ter}}(x)\\
	\vdots \\
	w(N-1,N^*) f_{\text{ter}}(x)
	\BOUT \right\Vert^2
	\label{eq:topmls}
\end{align}
We assume that $N > \hat{N}^*$ is reasonably chosen such that a feasible solution to the TOC exists. Therefore, in the following we  only consider the unconstrained version of~\eqref{eq:topmls} (we could argue that for example controls at the limits are simply removed from the optimization problem in the sense of null-space methods~\cite{pfeiffer2023b}).
The corresponding first order optimality conditions $K_{x,u,N^*}\coloneqq\nabla_{x,u,N^*} \mathcal{L}=0$ derived from the Lagrangian $\mathcal{L} \coloneqq \frac{1}{2}\Vert \BIN N^{*}\Delta t & \cdots & w(N-1,N^*) f_{\text{ter}}(N-1)^T \BOUT^T \Vert^2$ are
\begin{align}
&\BIN
K_x \\
K_u \\
K_{N^*}
\BOUT
 \hspace{-2.5pt}\coloneqq \hspace{-2.5pt}
\BIN
 \sum_{i=0}^{N-1}w(i,N^*)^2f_{\text{ter}}(x(i))^T\nabla_x f_{\text{ter}}(x(i))\\
\sum_{i=0}^{N-1}w(i,N^*)^2f_{\text{ter}}(x(i))^T\nabla_u f_{\text{ter}}(x(i))\\	
	N^{*}\Delta t^2 +  \Sigma(x,N^*)
	\BOUT = 0\nonumber
\end{align}
$\Sigma(x,N^*)$ is defined as
\begin{equation}
	\Sigma(x,N^*)  \coloneqq \sum_{i=0}^{N-1} 
	w(i,N^*)\nabla_{N^*}w(i,N^*) \Vert f_{\text{ter}}(x(i))\Vert^2  \label{eq:sigma}
\end{equation}

We now proceed as follows. We first show that the solution $N^*$ is contained within the horizon $0\leq N^* \leq N-1$, see Sec.~\ref{sec:cohn}. This is important for algorithm coherence since a negative $N^*$ would be an irrational solution. Similarly, $N^* > N-1$ would mean that the terminal constraint $f_{\text{ter}}$ vanishes from the optimization problem. We then proceed by showing the convergence of ADTOC  (Tab.~\ref{tab:topmaprx}) to \ref{eq:topm} in the limit and for resting goal points, see Sec.~\ref{sec:convadTOC}.

\subsection{Coherent solution $0\leq N^* \leq N-1$}
\label{sec:cohn}

We can make following statement about $\Sigma(x,N^*)$:

\begin{theorem}
	\label{th:neg}
The sum $\Sigma(x,N^*)<0$ is negative for any $N^*\in\mathbb{R}$ and $k=k_{\epsilon}$.
\end{theorem}
\begin{proof}
We look at the expression
\begin{align}
	&\sum_{i=0}^{N-1} w(i,N^*)\nabla_{N^*}w(i,N^*) = \sum_{i=0}^{N-1} (i-N^*+1)^{2k}\\
	& \left(h(i,N^*)\nabla_{N^*}h(i,N^*) - \frac{kh(i,N^*)^2}{i-N^*+1}\right) < 0\nonumber
\end{align}
The first term is negative since $(i-N^*+1)^{2k}>0$, $h(i,N^*)>0$ and $\nabla_{N^*}h(i,N^*) < 0$. The second term is positive for all $i < N^*-1$ since we have $(i-N^*+1)^{2k-1} < 0$ on this interval. However, with the assumption $k=k_{\epsilon}$, $h(i<N^*-1, N^*)^2 < \epsilon$ can be neglected for $i<N^*-1$, effectively rendering the second term negative as well.
\end{proof}
We show in the evaluation (Sec.~\ref{sec:eval}) that we achieve good convergence even with rather large $\epsilon$ (and therefore not strict negativity of $\Sigma(x,N^*)$)

Importantly, due to the additional term $(i-N^*+1)$ in $w$ any $f_{\text{ter}}(x(i))\neq 0$ with $i> N^*$ contributes to the sum $\Sigma(x,N^*)$. This would not be the case for the pure heaviside function with $k=k_{\epsilon}$. Here, any Newton step $K_{N^*}(u + \Delta u)$, $K_{N^*}(x + \Delta x)$ or $K_{N^*}(N^* + \Delta N^*)$ would not lead to reduction of the error $w(i,N^*)f_{\text{ter}}(x(i))\neq 0$ with $i> N^*$.

The choice of the factor $1$ in $(i-N^*+1)^{k}$ ensures that the crossing of $h(N^*,N^*)(N^*-N^*+1)=0.5$ at $i=N^*$ is the same as the original desired function $h(N^*,N^*) = 0.5$ (see Fig.~\ref{fig:weight} for $h\cdot(i-N^*)^k$ as a counterexample). 
%fully positive over the range $0 \leq N^* \leq N$ (for example by $(i-N^*+N)$) is that we would create a function without a crossing at $(N^*,0.5)$ . With our chosen value~1 

We now make a statement about the range of the obtained solution $N^*$.

\begin{theorem}
	\label{th:pos}
	The obtained $N^*$ is within the range $0\leq N^*\leq N-1$ for $k = k_{\epsilon} \geq 1$.
\end{theorem}

\begin{proof}
The Newton step $\Delta N^*$ with respect to $N^*$ is
\begin{align}
		&K_{N^*}(N^* + \Delta N^*) \approx K_{N^*}(N^*) + \Delta N^* \nabla_{N^*}K_{N^*}(N^*) =\nonumber\\
		&  N^*\Delta t^2 + \Sigma(x,N^*) + \Delta N^*\bigg(\Delta t^2 + \sum_{i=0}^{N}\label{eq:newtonn}\\
		& \left(\nabla_{N^*}w(i,N^*)^2 + w(i,N^*)\nabla_{N^*}^2w(i,N^*)\right)\Vert f_{\text{ter}}(x(i)) \Vert^2\bigg)\nonumber
\end{align}
For $k = k_{\epsilon} \geq 1$ and $N^*\leq 1$, the above sum is positive.
First, for $k = k_{\epsilon}$ any derivative of $h$ vanishes for $i\geq0$ and $N^*<0$.
We then have 
$\nabla^2_{N^*} w(i,N^*) =  h(i,N^*) k(k-1)    (i-N^*+1)^{k-2} \geq  0$
%$\nabla^2_{N^*} w(i,N^*) = 0.5 k    (1 - \tanh(k(i - N^*))^2) >  0$ 
 which is positive for $i\geq 0$, $N^*\leq1$ and $k\geq 1$.
The Newton step with respect to $N^*$ then becomes
\begin{align}
&\Delta N^* = 
\frac{-N^*\Delta t^2 - \sum(x,N^*)}{  \Delta t^2 +
	\sum_{i=0}^{N}
%	\left(\nabla_{N^*}w(i)^2\hspace{-2pt} +\hspace{-2pt} w(i)\nabla_{N^*}^2w(i)\right)
		\left(\dots\right)
	\Vert f_{\text{ter}}(x(i)) \Vert^2}\hspace{-2pt} >\hspace{-2pt} 0\nonumber
\end{align}
with both positive numerator (using theorem~\ref{th:neg} for $\sum(x,N^*)<0$) and denominator.
This means that for any negative $N^*$ we get a positive Newton step $\Delta N^*$ until $N^* + \Delta N^* \geq 0$. For $N^* > N-1$ all sums in~\eqref{eq:newtonn} vanish entirely. The resulting  Newton step $\Delta N^* = -N^*< 0$ is negative until $N^* \leq N-1$.  
\end{proof}

The latter ensures that the constraint $f_{\text{ter}} = 0$ never entirely vanishes from the optimization problem.

\subsection{Convergence to the true discrete time-optimal control}
\label{sec:convadTOC}

We first make a statement about the possible set of KKT points within the given control horizon.
\begin{theorem}
		\label{th:kktpoint}
	Any $0\leq N^*\leq N-1$ poses a KKT point to ADTOC (Tab.~\ref{tab:topmaprx}) if any $\Vert f_{\text{ter}}(x(i))\Vert^2 > 0$, with $i\geq N^*$ and some $x$, poses a feasible point to the constraint polytope $\Omega$.
\end{theorem}

\begin{proof}
	Since $N^*\Delta t^2>0$ and $\Sigma(x,N^*) < 0$,  we get $K_{N^*}=0$ for any $N^*$ if
	we can find a corresponding $x$, $u\in\Omega$ (and therefore $\Vert f_{\text{ter}}(x)\Vert$).
\end{proof}

With this foundation let's look at the overall convergence behavior of ADTOC  (Tab.~\ref{tab:topmaprx}).

\begin{theorem}
	\label{th:conv}
	ADTOC  (Tab.~\ref{tab:topmaprx}) represents \ref{eq:topm} if $\Delta t \rightarrow 0$, $k = k_{\epsilon}$ and the desired point $f_d$ is a resting point such that $\Vert f_{\text{ter}}(i> \hat{N}^*)\Vert^2 = 0$.
\end{theorem}

\begin{proof}
	Assume the time-optimal control $\hat{u}$ and state $\hat{x}$ such that $\Vert f_{\text{ter}}(i> \hat{N}^*)\Vert^2 = 0$ (which implies a resting goal point $f_d$). 

First, we consider the case $N^* \geq \hat{N}^*+1$. Then the sum $\Sigma_{k_{\epsilon}}=0$~\eqref{eq:sigma} vanishes for all $N^* \geq \hat{N}^*+1$  since $\Vert f_{\text{ter}}(i> \hat{N}^*)\Vert^2 = 0$. The resulting Newton step~\eqref{eq:newtonn} is negative  $\Delta N^* < 0$ and drives $N^*\rightarrow \hat{N}^*+1$. 
	
Secondly, we consider the case $N^* \leq \hat{N}^*$.
The corresponding Lagrangian becomes $\mathcal{L}({N}^*) = 0.5({N}^{*}\Delta t)^2
+ \sum_{i=0}^{N^*}w(i,N^*)^2\Vert f_{\text{ter}}(x(i)) \Vert^2$. We now consider the limit case $\Delta t \rightarrow 0$. We then have $\mathcal{L}({N}^*) = \sum_{i=0}^{N^*}w(i,N^*)^2\Vert f_{\text{ter}}(x(i)) \Vert^2> 0$.
Furthermore, for $N^* = \hat{N}^*$ we have $\mathcal{L}(\hat{N}^*) = (\hat{N}^*\Delta t)^2$ and limit $\mathcal{L}(\hat{N}^*) \rightarrow 0$. Since $\mathcal{L}(N^*)\geq 0$ for all $N^*\in\mathbb{Z}_{\geq0}$ this is a global minimum $\min{\mathcal{L}}$. This means that a global KKT point is within the interval $N^*\in]\hat{N}^*,\hat{N}^*+1[$ since $\min{\mathcal{L}}$. Due to theorem~\ref{th:kktpoint} we therefore find an optimal point $\hat{x}, \hat{u}\in\Omega$ if it exists.

Now assume a sub-optimal (s) $\hat{u}_s$ and $\hat{x}_s$ that is optimal to $\hat{N}^*_s > \hat{N}^*$ with $\Vert f_{\text{ter}}(i > \hat{N}^*_s)\Vert^2 = 0$ but sub-optimal to $\hat{N}^*$. We then clearly have for the sub-optimal Lagrangian $\mathcal{L}_s(\hat{N}^*_s) = \hat{N}^*_s\Delta t^2 > \mathcal{L}(\hat{N}^*) =  \hat{N}^*\Delta t^2$. This means that the true time-optimal $\hat{N}^*$ poses a global optimum to ADTOC  (Tab.~\ref{tab:topmaprx}).

\end{proof}

Note that in the above we do not make any assumptions about the structure of $f_{\text{dyn}}$ or $f_{\text{ter}}$. Our algorithm is therefore applicable to both linear and non-linear systems as we demonstrate in the evaluation Sec.~\ref{sec:eval}. In the non-linear case, only a sub-optimal time-optimal control $\hat{N}_s > \hat{N}$ may be identified depending on the (local or global) convergence properties of the used optimizer.

\section{Evaluation}
\label{sec:eval}

We apply our method to both a linear point-mass  and a non-linear robot manipulator control example. 
In order to solve the \ref{eq:NL-HLSP}s, we use the sequential hierarchical least-squares programming solver S-HLSP with trust-region and hierarchical step-filter proposed in~\cite{pfeiffer2023b}. It is characterized by global convergence to a local KKT point. S-HLSP is based on the sparse HLSP solver s-$\mathcal{N}$IPM-HLSP~\cite{pfeiffer2021,pfeiffer2023b}. All gradients and Hessians of the dynamics and task functions are computed analytically, for example according to~\cite{carpentier2018,Jong2021}. We choose $\Delta t = 0.1$~s, $N = 25$ (or $\Delta t=0.01$~s with $N=100$ and $N=80$) and $k=4$ with $\epsilon = 3.4\cdot 10^{-4}$. While $\epsilon$ is relatively large (meaning that $\Sigma(x,N^*)$ can become positive, in contrast to theorem~\ref{th:neg}) we did not experience any negative influence on the algorithm convergence. 

\begin{table}[htp!]
	\centering
	\setlength{\tabcolsep}{1pt}
	\setlength{\extrarowheight}{1pt}
	\begin{tabular}{@{} ccccc @{}}  
		\toprule
		$l$ & $f_l(x,u) \leqq v_l$ \\
		\midrule
		& with $i = 0,\dots,N-1$\\
		1 & 
		$\BIN 
		u_i - u_{\max}\\
		u_{\min} - u_i\\
		\BOUT
		\leq v_{1,i+1}$ \\ 
		1 & 	$f_{\text{dyn}}(x_0,u_0,...,u_{N-1},x_{N}) = v_1 $\\
		{\color{blue}2} & 
		{\color{blue}$f_{\text{ter}}(x(N^*+1),x)=v_2$}\\
		{\color{red}2} & {\color{red}$\BIN 
		f_{\text{ter}}(x(N^*+1),x)\\
		\vdots \\
		f_{\text{ter}}(x(N),x)
		\BOUT=v_2$}\\
		3 & Regularization of $x$, $u$\\
		\bottomrule
	\end{tabular}
	\caption{DTOC (blue) and `padded' DTOC (PDTOC: red). $N^*$ is fixed.}
	\label{tab:topmaprxtest}
\end{table}

In order to confirm our results we compare our algorithm to the two optimization problems given in Tab.~\ref{tab:topmaprxtest}. 
Both problems' decision variables do not include $N^*$ which is rather fixed at a chosen value.
The first problem therefore corresponds to \ref{eq:topm} with the terminal constraint at $N^*+1$.
The other optimization problem is a `padded' version of DTOC (PDTOC) with terminal constraints on the collocation points $N^*+1$ to $N$. This reproduces the behavior of ADTOC  (Tab.~\ref{tab:topmaprx}) which can solve the original problem~\eqref{eq:topm} equivalently only if the desired state $f_d$ is a resting goal point (see theorem~\ref{th:conv}). Furthermore, we report the results of GPOPS~\cite{gpops} which implements a direct transcription method for optimal control problems of the form~\eqref{eq:conttopm}. The resulting NLP is solved by the local optimizer SNOPT~\cite{snopt}.

In order to preserve sparsity we set $w$ to zero when it is close to 0 up to a numerical threshold
\begin{equation}
	w(i,N^*) \leftarrow
	\begin{cases}
		0 & \text{if $w(i,N^*) < 10^{-20}$}\\
		w(i,N^*) & \text{otherwise}
	\end{cases}
\end{equation}
Note that we do not cut values that are close to 1 since this would reintroduce the discontinuity that we are trying to circumvent to begin with.

We display the variable $T^* = N^* \Delta t$ which indicates the optimal time at which still $f_{\text{ter}} \neq 0$ (the time ($N^*+1) \Delta t$ is the time where we first have $f_{\text{ter}} = 0$).

\subsection{Point mass}

\begin{figure}[htp!]
	\vspace{5pt}	
	\includegraphics[width=0.9\columnwidth]{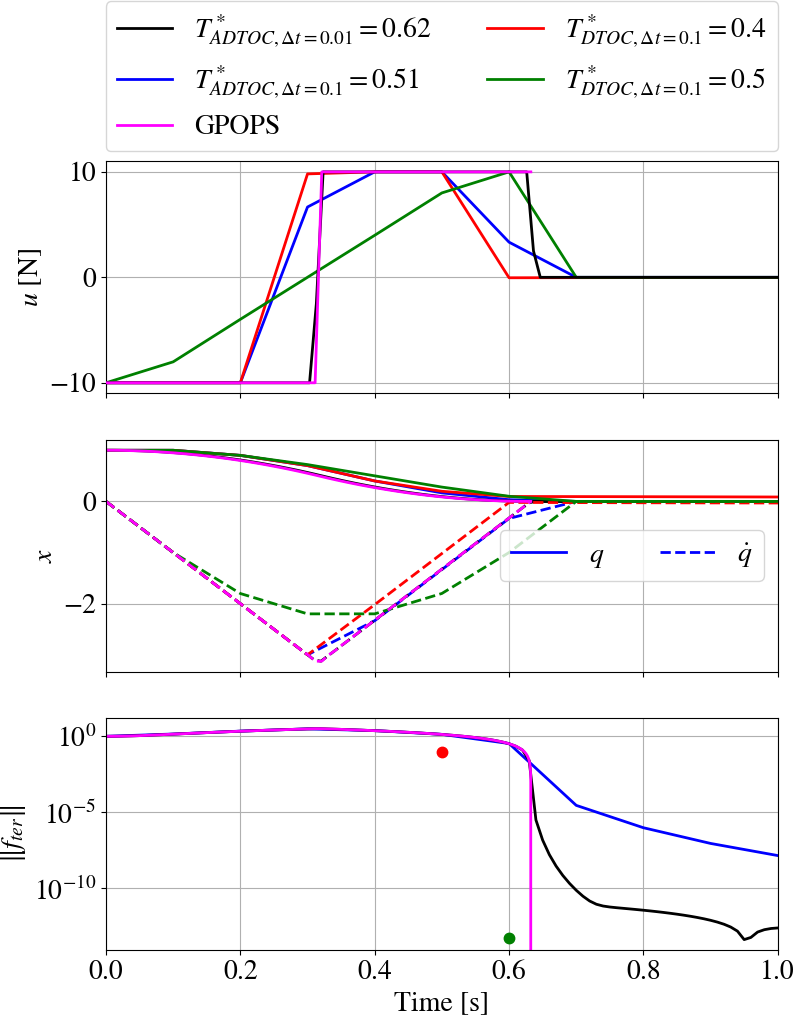}
	\centering
	\caption{Control $u$, states $x$ and norm of task error $\Vert f_{\text{ter}} \Vert$ for a linear point mass moving on a plane for ADTOC  (Tab.~\ref{tab:topmaprx}) ($T^*_{\text{ADTOC}}$) and DTOC with fixed $N^*$ (Tab.~\ref{tab:topmaprxtest}) ($T^*_{\text{DTOC}}$).}
	\label{fig:pm}
\end{figure}

First, we apply our method on a simple point mass of $m=1$~kg moving along a horizontal line with single-input force $u$ and state $x = \BIN q & \dot{q} \BOUT$. The state $q$ is the position of the point-mass. The state $\dot{q}$ is the point-mass velocity. The initial state is chosen as $x_0 = \BIN 1 & 0\BOUT$. The point-mass dynamics are linear in the state $x$ and control $u$ (using explicit Euler integration $x_{i+1} = x_i + \Delta t \dot{x}_i$ and $m\ddot{q}_{i} = u_i$ with $i=0,\dots,N-1$)
\begin{align}
	f_{\text{dyn}}(i) &=
	x_{i+1}
	+
	\BIN
	-1 & -\Delta t \\
	0 & -1
	\BOUT
	x_i
	-
	\BIN 
	0 \\
	\Delta / m
	\BOUT
	u_i
\end{align}
The control limit is chosen to $10$~N. We set our task function to $f_{\text{ter}} = x$ with $f_d = 0$. 

The results of ADTOC  (Tab.~\ref{tab:topmaprx}) are given in Fig.~\ref{fig:pm}. $N^*$ is identified as $N^*=5.1$ for $\Delta t=0.1$~s and $N^*=62$ for $\Delta t = 0.01$~s. As can be seen from the top graph, the identified control trajectories ($T^*_{\text{ADTOC},\Delta t = 0.1}=0.51$ and $T^*_{\text{ADTOC},\Delta t = 0.01}=0.62$) follow a bang-bang profile which indicates time-optimality. For $\Delta = 0.01$~s (computation time 4.5~s with 75 solver iterations, $N=100$) the control is at the limit $-10$~N until control time 0.3~s and then at the limit $10$~N from time 0.33~s. The control is zero from time 0.65~s onwards. A similar but less pronounced profile is observed for $\Delta t=0.1$~s (computation time 0.7~s with 64 solver iterations, $N=25$), seemingly confirming our convergence theorem~\ref{th:conv} of equivalence of ADTOC  (Tab.~\ref{tab:topmaprx}) with \ref{eq:topm} for $\Delta t \rightarrow 0$ (note that $\Delta t\rightarrow 0$ is prohibitive due to memory limitations since $N \rightarrow \infty$).

For comparison, we give the results of GPOPS (100 transcription nodes, computation time 0.88~s and 19 solver iterations). The corresponding control profile is very similar to that of our method with $\Delta t=0.01$~s but with a slightly sharper switching point from control on the lower bound $-10$~N to the upper bound $10$~N. This leads to a sharp drop of $\Vert f_{\text{ter}}\Vert$ at $t = 0.63$~s. On the other hand, the task error for our method $T^*_{\text{ADTOC},\Delta t = 0.01}=0.62$ stays elevated until 0.62~s but then declines sharply to $\Vert f_{\text{ter}} \Vert \approx 10^{-11}$. For $\Delta t \rightarrow 0$ (and therefore \ref{eq:topm} according to theorem~\ref{th:conv}) we would expect a behavior similar to GPOPS with a finite jump of $\Vert f_{\text{ter}} \Vert$ to zero between the corresponding control iterations $\hat{N}^*$ and $\hat{N}^*+1$. 

Furthermore, we give the results of DTOC (Tab.~\ref{tab:topmaprxtest}) with fixed $N^*=4$ and $N^*=5$. For $N^*=4$ ($T^*_{\text{DTOC},\Delta t = 0.1}=0.4$) the control profile is sharper than for $T^*_{\text{ADTOC},\Delta t = 0.1}=0.51$. The control then drops to zero at time 0.7~s. However, from the bottom graph it can be observed that the task error is only reduced to $\Vert f_{\text{ter}}\Vert = 0.1$. This is in contrast to the results for $T^*_{\text{DTOC},\Delta t = 0.1}=0.5$ where the task error is reduced to $\Vert f_{\text{ter}}\Vert = 10^{-13}$ at time 0.6~s. However, the corresponding control profile is very conservative and does not follow a bang-bang profile. A good middle ground between time optimality and error reduction is identified for ADTOC  (Tab.~\ref{tab:topmaprx}) $T^*_{\text{ADTOC},\Delta t = 0.1}=0.51$ (while not being as computationally expensive as $T^*_{\text{ADTOC},\Delta t = 0.01}=0.62$). The task error at time 0.6~s is approximately of order $\Vert f_{\text{ter}}\Vert = 5\cdot 10^{-4}$ which can be considered sufficiently accurate for this control application.

The task error for $T^*_{\text{DTOC},\Delta t = 0.1}=0.5$ declines the fastest with $\Vert f_{\text{ter}} \Vert \approx 10^{-13}$ at time 0.6~s. However, while this is closest to the desired solution of~\ref{eq:topm}, it does not fulfill higher order constraints like zero acceleration or jerk at convergence. These are implicitly fulfilled for ADTOC  (Tab.~\ref{tab:topmaprx}) and PDTOC with fixed $N^*$ in Tab.~\ref{tab:topmaprxtest}, depending on how much bigger $N > N^*$. Note that this is a matter of defining a more embracing task function $f_{\text{task}}$ including acceleration and jerk regularization. On the other hand, due to theorem~\ref{th:conv} requiring a resting goal point $f_d$, ADTOC  (Tab.~\ref{tab:topmaprx}) is not able to achieve the same results as DTOC with fixed $N^*$ (Tab.~\ref{tab:topmaprxtest}).

\subsection{Planar manipulator with two joints}

\label{sec:manip}

\begin{figure}[htp!]
	\vspace{5pt}
	\includegraphics[width=0.9\columnwidth]{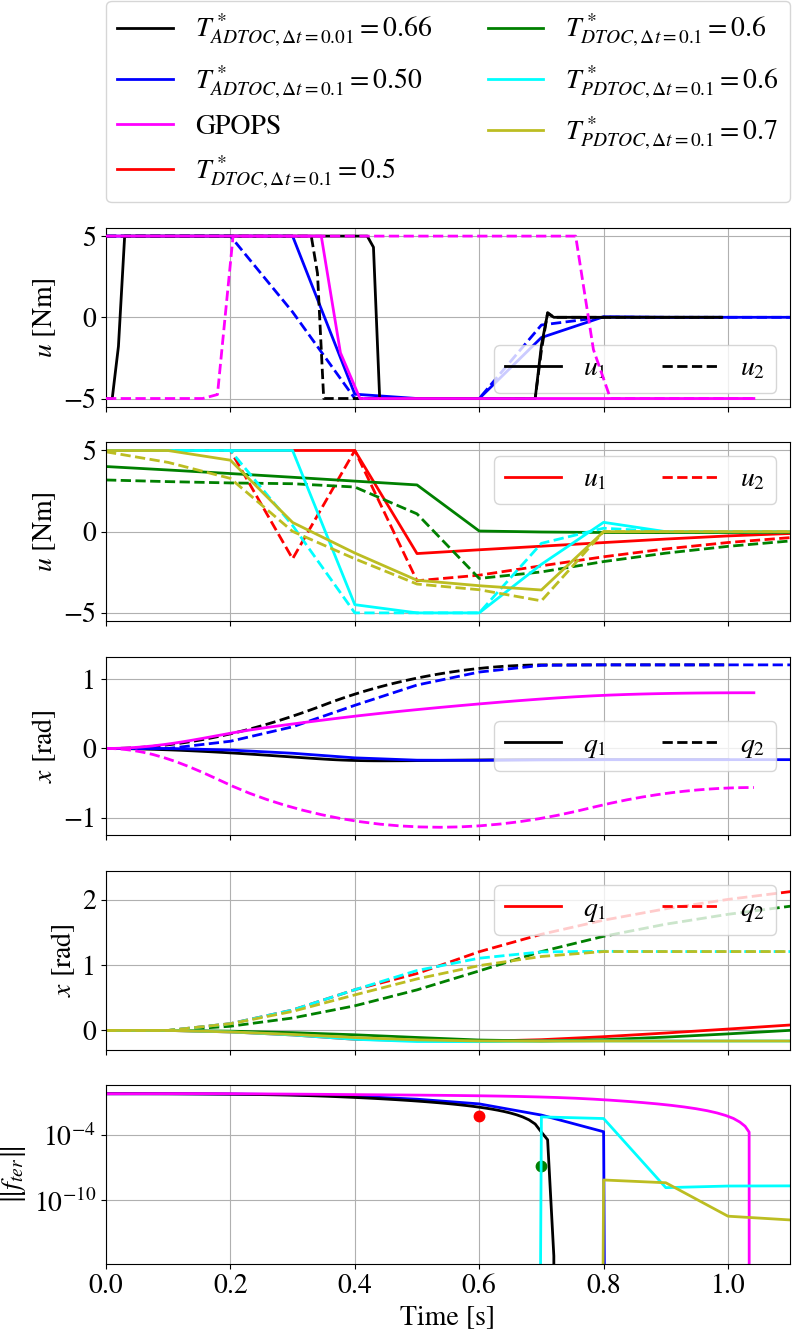}
	\centering
	\caption{Controls $u$, states $x$ and norm of task error $\Vert f_{\text{ter}} \Vert$ for a non-linear 2D manipulator with two joints for ADTOC  (Tab.~\ref{tab:topmaprx}), DTOC and PDTOC (Tab.~\ref{tab:topmaprxtest}).}
	\label{fig:manip2d}
\end{figure}

In this section we want to identify the time-optimal control for moving a 2D planar manipulator with fixed base and two joints from its initial task space position $f_{\text{task}}(x) = \BIN 2 & 0\BOUT^T$~m to $f_d = \BIN 1 & 1 \BOUT^T$~m. 
The robot's link lengths are given by $L_1 = 1.25$~m and $L_2 = 0.75$~m. The robot state is given by $x = \BIN q_{0,0} & q_{1,0} & \dot{q}_{0,0} & \dot{q}_{1,0} &\cdots & \dot{q}_{0,T-1} & \dot{q}_{1,T-1}\BOUT$ where $q_0$ and $q_1$ are the joint angles of joint one and two, respectively. The initial state is given by $x_0 = \BIN 0 & 0 & 0 & 0 \BOUT$.
Both joints are actuated by the torques $u=\BIN \tau_1 & \tau_2\BOUT^T$. The joint torques are limited to $5$~Nm.
We use the Newton-Euler equations as described in~\cite{Baccouch2020} with link masses $m_1 = m_2 = 1$~kg for the Euler explicit integration scheme $x_{i+1} = x_i + \Delta t\dot{x}_i$. Due to the more generic form of the dynamics of~\ref{eq:topm} compared to the continuous one~\eqref{eq:conttopm}, we do not rely on matrix inversion in order to obtain an explicit formulation of the joint accelerations $\ddot{q}_1$ and $\ddot{q}_2$. Instead, we use the dynamics $f_{\text{dyn}} = M (\dot{q}_{i+1} - \dot{q}_i) + \Delta t\{M\ddot{q}_i\}$ where $q \coloneqq \BIN q_1 & q_2\BOUT^T$, $M$ is the joint space inertia matrix and the vector $\{M\ddot{q}_i\}$ represents the inverse dynamics equations~\cite{carpentier2018}.

The results are given in Fig.~\ref{fig:manip2d}. Our algorithm ADTOC  (Tab.~\ref{tab:topmaprx}) with $\Delta t = 0.01$~s ($T^*_{\text{ADTOC},\Delta t = 0.01}=0.66$) recovers a sharp bang-bang control profile for both $\tau_1$ and $\tau_2$ (computation time 22~s with 121 solver iterations, $N=80$). The control vanishes at approximately 0.72~s with $\Vert f_{\text{ter}} \Vert\approx 10^{-15}$. This is in contrast to $\Delta t = 0.1$~s ($T^*_{\text{ADTOC},\Delta t = 0.1}=0.57$, computation time 1.5~s with 88 solver iterations, $N=25$) with the control and task error  vanishing only at approximately 0.8~s.

Next, we report the results for GPOPS with 80 collocation points computed in 7.55~s and 28 solver iterations. We have $T^*=1.04$~s (the last time with non-zero control and task error). As can be seen from the corresponding joint trajectory, a different inverse kinematics solution to the boundary condition $f_{\text{ter}} = 0$ is found which is less time-optimal than the one found for our method (note that there is an infinite number of solutions due to robot redundancy~\cite{Siciliano1991}). This is due to the non-linear optimizer not finding a better solution to the problem. Nonetheless, the bang-bang control profile is time-optimal with respect to the identified joint trajectory.

Furthermore, the results of the problem Tab.~\ref{tab:topmaprxtest} with fixed $N^*=5$ and $N^*=6$ are reported. For $N^*=5$ ($T^*_{\text{DTOC},\Delta t = 0.1}=0.5$) we have a limit control at the upper bound 5~Nm. However, the task error at 0.6~s is only decreased to $7\cdot 10^{-3}$. This is in contrast to the control profile of $T^*_{\text{DTOC},\Delta t = 0.1}=0.6$ with the more conservatively chosen terminal point $N^*=6$. There is enough time to reduce the error to $2\cdot10^{-7}$. However, from the conservative control profile it can be concluded that $N^*=6$ is not time-optimal. Unlike ADTOC  (Tab.~\ref{tab:topmaprx}) and PDTOC with known $N^*$ (Tab.~\ref{tab:topmaprxtest}), acceleration and jerk constraints are not implicitly fulfilled for DTOC with known $N^*$ (Tab.~\ref{tab:topmaprxtest}) and therefore enables quicker task convergence. This means that the point $f_d$ is reached only instantaneously at times 0.6~s and 0.7~s with divergence afterwards, as can be seen from the non-stationary state trajectories in the second graph from the bottom. 

The results of PDTOC (Tab.~\ref{tab:topmaprxtest}) with fixed $N^*=6$ and $N^*=7$ follow closely the ones of our method ADTOC (Tab.~\ref{tab:topmaprx}) with $\Delta t = 0.1$~s in terms of controls, states and task error reduction. Thereby, the same pattern as with the results for  Tab.~\ref{tab:topmaprxtest} is observed: for too low $N^*=6$ a bang-bang control profile is obtained but the task error is reduced insufficiently ($\Vert f_{\text{ter}}\Vert = 5\cdot 10^{-3}$); for too high $N^*=7$ the task error is reduced significantly ($\Vert f_{\text{ter}}\Vert = 8\cdot 10^{-9}$) but the control profile is not time-optimal / not a limit one.

\section{Conclusion}

In this work we have formulated, implemented and evaluated an easily implementable \ref{eq:NL-HLSP} for ADTOC. We showed that the method corresponds to the true DTOC in the limits and for resting goal points. This behavior was confirmed in simulations with linear and non-linear systems, recovering the bang-bang control profiles typically seen in TOC.

In future work we would like to investigate whether the current limitation to resting goal points can be relaxed for a broader applicability of our method. Furthermore, while we demonstrated the applicability of our method to whole-body motion control, it is limited by its computational complexity. We therefore would like to investigate methods to reduce the computational burden, especially with respect to large robotic systems like humanoid robots.

%Furthermore, the proposed method is currently limited by its computational complexity, especially with regards to robot whole body optimization like in Sec.~\ref{sec:manip}. To the best of our knowledge there are currently no specific methods to efficiently compute the gradient and Hessians of system dynamics with respect to their state and control variables. For future work it would be therefore desirable to design efficient methods to compute system derivatives, for example based on the Newton-Euler equations. This should include higher order derivatives of variables that have been substituted for example by finite differences.

\section{Acknowledgments}

We would like to thank Dr. Bing Song for her insightful comments on our work.

	\balance
	\bibliographystyle{IEEEtran}
	\bibliography{bib}

\end{document}

%% file: commands.tex
\DeclareMathOperator*{\mini}{min.}

\definecolor{light-gray}{gray}{0.95}
\definecolor{dark-gray}{gray}{0.5}
\definecolor{mygray}{gray}{0.75}
\newcommand{\BIN}{\begin{bmatrix}}
\newcommand{\BOUT}{\end{bmatrix}}

\newcommand{\uf}{\underline{f}}

\newcommand{\uv}{\underline{v}}

\pgfdeclarelayer{bg1}   
\pgfdeclarelayer{bg}  
\pgfsetlayers{bg1,bg,main}  

\definecolor{orange}{rgb}{0.99,0.69,0.07}
\definecolor{lightgray}{gray}{0.85}
\definecolor{light-gray}{gray}{0.95}
\definecolor{dark-gray}{gray}{0.5}

\tikzset{cross/.style={cross out, draw=black, minimum size=2*(#1-\pgflinewidth), inner sep=0pt, outer sep=0pt},
cross/.default={1pt}}

 \makeatletter
 \newcommand\fs@spaceruled{\def\@fs@cfont{\bfseries}\let\@fs@capt\floatc@ruled
   \def\@fs@pre{\vspace{5pt}\hrule height.8pt depth0pt \kern2pt}%
   \def\@fs@post{\kern2pt\hrule\relax}%
   \def\@fs@mid{\kern2pt\hrule\kern2pt}%
   \let\@fs@iftopcapt\iftrue}
 \makeatother

\newtheorem{theorem}{Theorem}
\newtheorem{definition}{Definition}[section]